\documentclass[11pt,letterpaper]{article}
\usepackage[hmargin=2.5cm,vmargin=3.3cm]{geometry}
\usepackage{authblk}
\usepackage{booktabs} 
\usepackage{paralist}
\usepackage{url}
\usepackage{amsmath,amsthm}
\usepackage{bbm}
\usepackage[ruled]{algorithm}
\usepackage[]{algorithmic}
\usepackage{tabularx}
\usepackage{graphicx}
\usepackage[caption=false,font=footnotesize]{subfig}
\usepackage{multirow}
\usepackage{paralist}
\usepackage{booktabs}
\usepackage{epstopdf}
\usepackage{amssymb}
\usepackage{textcase}
\usepackage{indentfirst}
\usepackage{array}
\usepackage{capt-of}
\usepackage{xspace}
\usepackage{color}
\usepackage{amsfonts,mathrsfs,bm}
\usepackage{lipsum}
\usepackage{multirow}

\title{Non-Exhaustive, Overlapping Co-Clustering: An Extended Analysis}

\vspace{0.3cm}

\author[1]{Joyce Jiyoung Whang}
\author[2]{Inderjit S. Dhillon}
\affil[1]{Dept. of Computer Science and Engineering\\
Sungkyunkwan University\\
jjwhang@skku.edu}
\affil[2]{Dept. of Computer Science\\
The University of Texas at Austin\\
inderjit@cs.utexas.edu}

\providecommand{\keywords}[1]{\textbf{\textit{Index Terms---}} #1}

\date{}

\newtheorem{mythm}{Theorem}
\newtheorem{mylem}{Lemma}

\DeclareMathOperator{\trace}{trace}

\def\clap#1{\hbox to 0pt{\hss#1\hss}}

\newcommand{\mat}[1]{\boldsymbol{#1}}

\providecommand{\mH}{\ensuremath{\mat{H}}}
\providecommand{\mI}{\ensuremath{\mat{I}}}

\providecommand{\mM}{\ensuremath{\mat{M}}}

\providecommand{\mU}{\ensuremath{\mat{U}}}
\providecommand{\mV}{\ensuremath{\mat{V}}}

\providecommand{\mX}{\ensuremath{\mat{X}}}

\newcommand\norm[1]{\left\lVert#1\right\rVert}

\begin{document}

\maketitle

\begin{abstract}
The goal of co-clustering is to simultaneously identify a clustering of rows as well as columns of a two dimensional data matrix. A number of co-clustering techniques have been proposed including information-theoretic co-clustering and the minimum sum-squared residue co-clustering method. However, most existing co-clustering algorithms are designed to find pairwise disjoint and exhaustive co-clusters while many real-world datasets contain not only a large overlap between co-clusters but also outliers which should not belong to any co-cluster. In this paper, we formulate the problem of Non-Exhaustive, Overlapping Co-Clustering where both of the row and column clusters are allowed to overlap with each other and outliers for each dimension of the data matrix are not assigned to any cluster. To solve this problem, we propose intuitive objective functions, and develop an an efficient iterative algorithm which we call the NEO-CC algorithm. We theoretically show that the NEO-CC algorithm monotonically decreases the proposed objective functions. Experimental results show that the NEO-CC algorithm is able to effectively capture the underlying co-clustering structure of real-world data, and thus outperforms state-of-the-art clustering and co-clustering methods. This manuscript includes an extended analysis of~\cite{whang-cikm2017a}.
\end{abstract}


\section{Introduction}
Clustering is one of the most important tools in unsupervised learning. Given a set of objects, each of which is represented by a feature vector, the goal of clustering is to group the objects into a certain number of clusters such that similar objects are assigned to the same cluster. Let us define a two dimensional data matrix such that each row of the matrix represents an object and each column represents an attribute or a feature of an object. In this setting, clustering algorithms are designed to cluster the rows of the data matrix. While most (one-way) clustering algorithms focus on clustering only one of the dimensions of the data matrix, it has been recognized that simultaneously clustering both dimensions of the data matrix is desirable to not only improve the clustering performance~\cite{dhillon-kdd03} but also detect more semantically meaningful clusters in many applications, e.g.,~\cite{daniel},~\cite{amela-bio06}, and~\cite{cheng-ismb00}.

The goal of co-clustering (also known as biclustering) is to simultaneously identify a clustering of the rows as well as the columns of the data matrix, i.e., group similar objects as well as similar attributes. The co-clustering problem has been studied in various applications such as gene expression data analysis~\cite{cheng-ismb00}, word-document clustering~\cite{dhillon01}, and market-basket analysis. Also, the co-clustering technique has been utilized in recommender systems~\cite{rec_icdm05},~\cite{rec_sigir05} where simultaneous clustering of users and items provides a way to develop a scalable collaborative filtering framework. 

\begin{figure}[t]
\centering
\begin{minipage}[b]{0.55\linewidth}\centering
  \centering
  \begin{tabular}{c}
    \subfloat[The original data matrix]{\includegraphics[width=\linewidth]{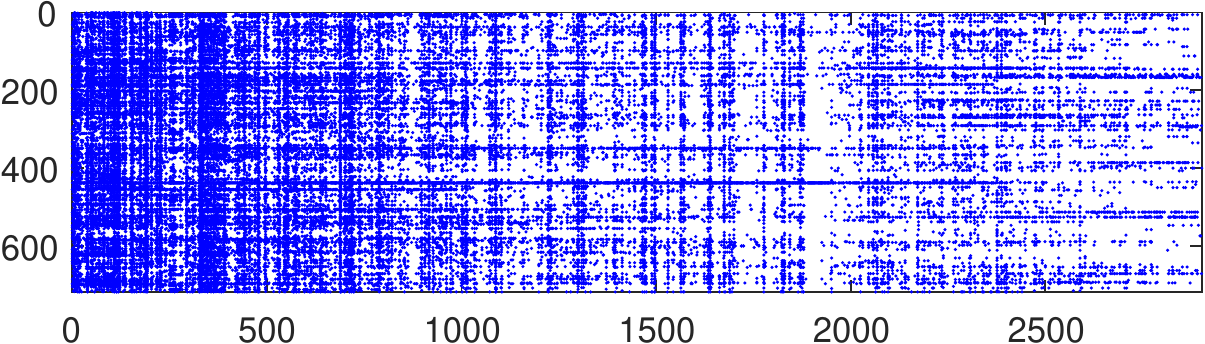}}  \\
    \subfloat[Rearrangement of the rows and the columns according to the output of the NEO-CC method]{\includegraphics[width=\linewidth]{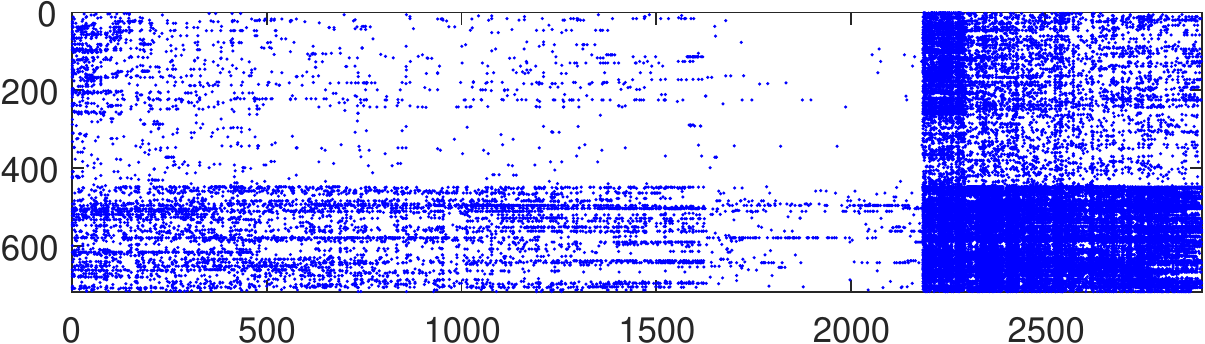}} 
\\
    \subfloat[The rearranged data matrix annotated with row and column clusters. The red and orange boxes indicate row clusters and the purple and green boxes indicate column clusters.]{\includegraphics[width=\linewidth]{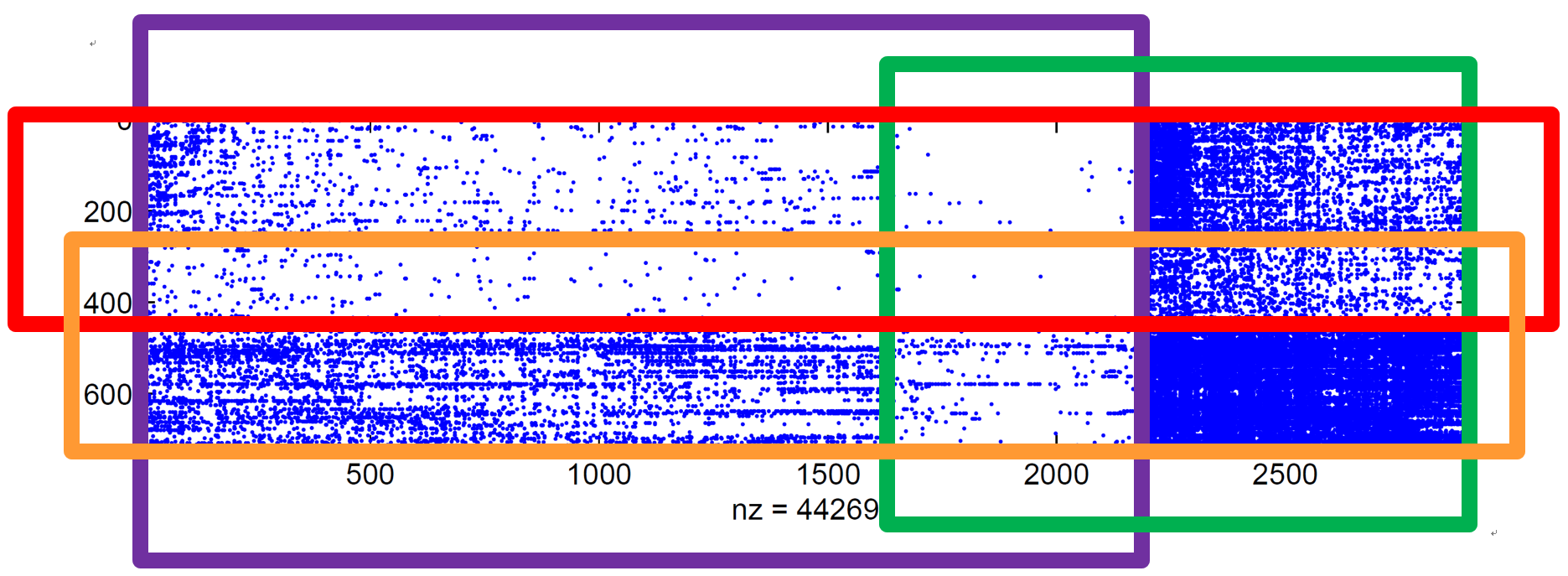}}
  \end{tabular}
\end{minipage}
\caption{Visualization of the co-clusters produced by the NEO-CC method.}
\label{fig:spy}
\end{figure}

Many different kinds of co-clustering methods have been proposed such as information-theoretic co-clustering~\cite{dhillon-kdd03}, the minimum sum-squared residue (MSSR) co-clustering~\cite{cho-sdm04}, and spectral co-clustering~\cite{dhillon01}. We note that most existing co-clustering methods are based on an assumption that every object belongs to exactly one row cluster and every attribute belongs to exactly one column cluster. However, this assumption hinders the existing co-clustering methods from correctly capturing the underlying co-clustering structure of data because in many real-world datasets, both of the row and column clusters can overlap with each other and there often exist outliers which should not belong to any cluster. 

For the one-way clustering problem, we recently proposed the NEO-K-Means method~\cite{whang-tpami2019} to identify overlapping clusters and outliers in a unified manner where the objective function and the algorithm were designed to cluster only the rows of a data matrix. Although the NEO-K-Means method has been shown to be effective in finding the ground-truth clusters~\cite{hou-whang-kdd2015}, how to extend this idea to the co-clustering problem is far from straightforward because of the complicated interactions between the rows and the columns of the data matrix where both the cluster overlap and the non-exhaustiveness are allowed for each dimension of the matrix. 

Inspired by the one-way NEO-K-Means clustering method~\cite{whang-sdm2015}, we propose the Non-Exhaustive, Overlapping Co-Clustering (NEO-CC) method to efficiently detect coherent co-clusters such that both of the row and column clusters are allowed to overlap with each other and the outliers for each dimension of the data matrix are not assigned to any cluster. Figure~\ref{fig:spy} shows the output of our NEO-CC method on a user-movie ratings dataset from~\cite{movie} (see Section~\ref{sec:exp} for more details about this dataset) where each row represents a user and each column represents a movie. In the data matrix, the non-zero elements are represented as blue dots. As shown in Figure~\ref{fig:spy}(a), when we simply visualize the original data matrix, it is hard to discover particular patterns of the matrix. In Figure~\ref{fig:spy}(b), we rearrange the rows and the columns according to the output of the NEO-CC method, and in Figure~\ref{fig:spy}(c), we explicitly annotate the row and column clusters. After we rearrange the rows and the columns, we can observe the co-clustering structure of the data matrix. The NEO-CC method discovers overlapping row and column clusters, and in our experiments, we notice that allowing the overlaps between the clusters significantly increases the accuracy of clustering (see the experimental results on ML1 dataset in Section~\ref{sec:exp}). Also, the NEO-CC method detects one outlier from the rows. When we look at the detected outlier, it corresponds to a user who randomly gives ratings to a number of movies without any particular pattern.

In this paper, we mathematically formulate the non-exhaustive, overlapping co-clustering problem, and propose two objective functions which we call the NEO-CC-M (NEO-CC by considering the Mean of co-clusters) objective and the NEO-CC-RCM (NEO-CC by considering the Row and Column Mean) objective. These objective functions seamlessly generalize the MSSR co-clustering~\cite{cho-sdm04} and the one-way NEO-K-Means clustering~\cite{whang-sdm2015} objectives. To optimize our NEO-CC objective functions, we propose a simple iterative algorithm called the NEO-CC algorithm. We theoretically show that the NEO-CC algorithm monotonically decreases the NEO-CC objective functions. Our objective functions and the algorithm allow us to solve the non-exhaustive, overlapping co-clustering problem in a principled way. Our experimental results show that the NEO-CC algorithm is able to effectively capture the underlying co-clustering structure of real-world data, and thus outperforms state-of-the-art clustering and co-clustering methods in terms of discovering the ground-truth clusters.

\section{Related Work}
There have been various approaches to solving the co-clustering problem. The minimum sum-squared residue (MSSR) co-clustering~\cite{cho-sdm04} is a well-known co-clustering method. The MSSR co-clustering method has been initially designed for analyzing gene expression data, but later it has been studied in a more general co-clustering framework~\cite{banerjee-jmlr07}. Information-theoretic co-clustering~\cite{dhillon-kdd03} also has been known to be one of the most successful co-clustering methods where the co-clustering problem is formed as an optimization problem of maximizing the mutual information between the clustered random variables. However, these co-clustering models have constraints on the structure of the co-clusters such that the co-clusters are pairwise disjoint and all the data points for each dimension should belong to some cluster.

We note that~\cite{bcc},~\cite{drcc} and~\cite{baier-cko97} study the overlapping co-clustering problem but do not consider outlier detection. On the other hand, ~\cite{trcc} proposes a robust co-clustering algorithm by assuming the presence of outliers in a dataset, but does not consider overlapping co-clustering. While these methods only consider either overlapping co-clustering or outlier detection, our NEO-CC algorithm simultaneously detects overlapping co-clusters as well as outliers. 

We note that the ROCC algorithm~\cite{deodhar-icml09} and an infinite plaid model (IPM)~\cite{infp} have been recently proposed to find non-exhaustive, overlapping co-clusters. However, the ROCC algorithm includes complicated heuristics, and the infinite bi-clustering method requires a user to provide many non-intuitive hyperparameters with the model. On the other hand, the NEO-CC algorithm includes simple and intuitive parameters which can be easily estimated. We compare the performance of our NEO-CC method with the ROCC and the IPM methods in Section~\ref{sec:exp} observing that the NEO-CC method significantly outperforms the ROCC and the IPM methods.

\section{The NEO-CC Objectives}
We formally define the non-exhaustive, overlapping co-clustering (NEO-CC) problem, and discuss how we can design reasonable objective functions for the NEO-CC problem. 

\subsection{Problem Statement}
Let us consider a two-dimensional data matrix $\mX \in \mathbb{R}^{n \times m}$. For each dimension of $\mX$, let $\mathcal{X}^r$ denote the set of data points for row clustering, and $\mathcal{X}^c$ denote the set of data points for column clustering. The co-clustering problem is to cluster $\mathcal{X}^r = \{{\bf x}_1,{\bf x}_2,\dots,{\bf x}_n\}$ (${\bf x}_i \in \mathbb{R}^m$ for $i=1,\dots,n$) into $k$ row clusters $\{\mathcal{C}^r_1,\mathcal{C}^r_2,\dots,\mathcal{C}^r_k\}$, and cluster $\mathcal{X}^c = \{{\bf x}_1,{\bf x}_2,\dots,{\bf x}_m\}$ (${\bf x}_j \in \mathbb{R}^n$ for $j=1,\dots,m$) into $l$ column clusters $\{\mathcal{C}^c_1,\mathcal{C}^c_2,\dots,\mathcal{C}^c_l\}$. 

Since the solution of the NEO-CC problem is allowed to contain overlaps among the clusters and exclude outliers from the clusters, we note that there might exist $i$ and $j$ ($i,j \in \{1,\dots,k\}$) such that $\mathcal{C}^r_i \cap \mathcal{C}^r_j \neq \emptyset$ for $i \neq j$, and $ \bigcup\limits_{i=1}^{k} \mathcal{C}_i^r \subseteq \mathcal{X}^r$. We can also show the similar derivations for the $l$ column clusters.

Given an element $x_{ij}$ in $\mX$, let $f({\bf x}_i)$ denote the set of row clusters that ${\bf x}_i$ belongs to ($i=1,\dots,n$), and let $g({\bf x}_j)$ denote the set of column clusters that ${\bf x}_j$ belongs to ($j=1,\dots,m$). We use the notations $f(\cdot)$ and $g(\cdot)$ for convenience even though we note that $f(\cdot)$ and $g(\cdot)$ are not functions of ${\bf x}_i$ or ${\bf x}_j$ because each data point is allowed to belong to more than one cluster.

\subsection{The NEO-CC-M Objective Function}

\begin{figure}[t]
\centering
\begin{minipage}[b]{0.7\linewidth}\centering
  \centering
  \begin{tabular}{cc}
    \subfloat[Data matrix $\mX$, row clustering $\mU$, and column clustering $\mV$]{\includegraphics[width=1\textwidth]{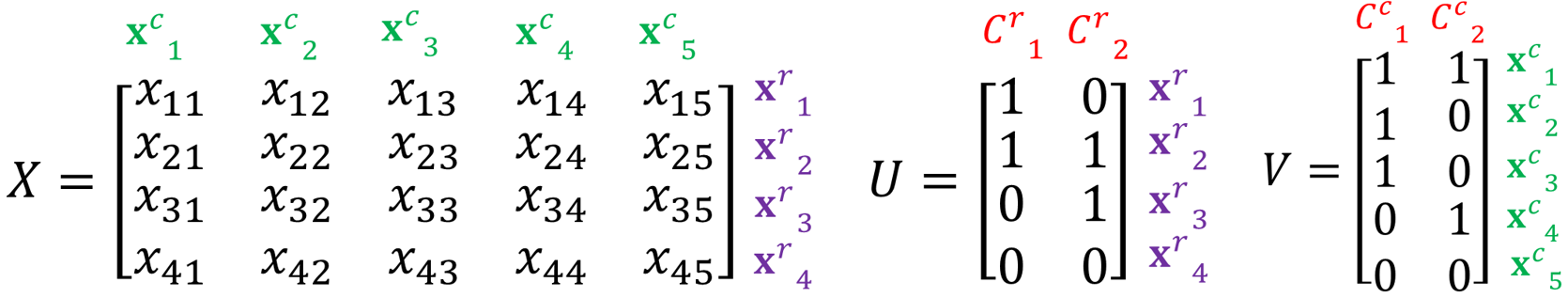}}\\
    \subfloat[The contribution of $x_{21}$ to the NEO-CC-M objective]{\includegraphics[width=0.9\textwidth]{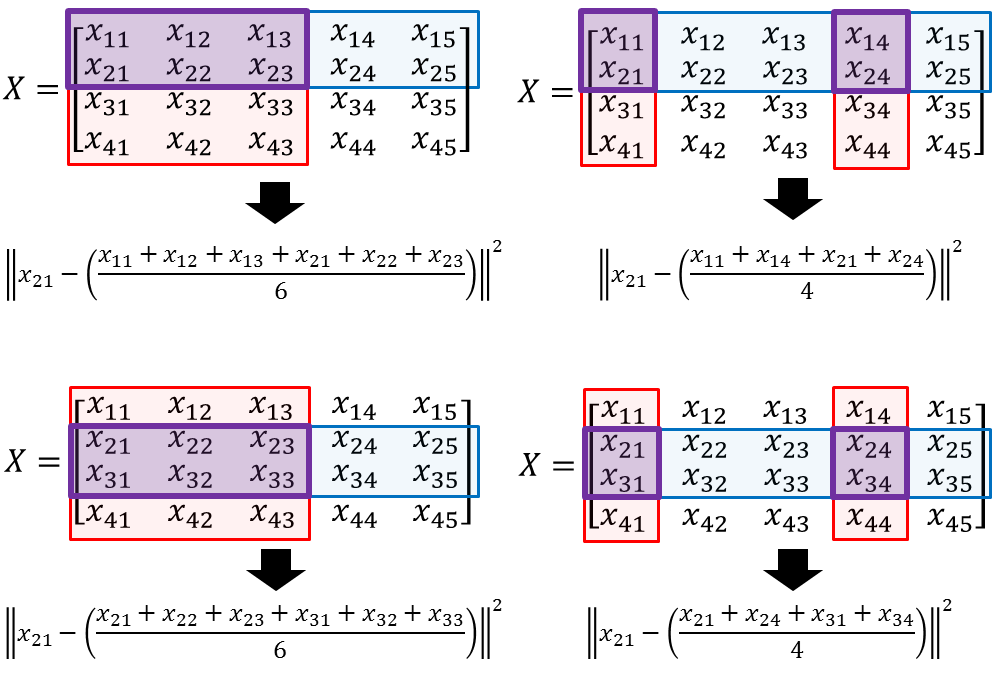}}
  \end{tabular}
\end{minipage}
\caption{Given a data matrix $\mX \in \mathbb{R}^{4 \times 5}$ and the assignment matrices $\mU$ and $\mV$, the contribution of an entry $x_{21}$ to the NEO-CC-M objective defined in~(\ref{neocc_obj}) is determined by $f({\bf x}^r_2)$ and $g({\bf x}^c_1)$. Note that ${\bf x}^r_2 \in \mathcal{C}^r_1$, ${\bf x}^r_2 \in \mathcal{C}^r_2$ (${\bf x}^r_2 \in \mathbb{R}^5$), ${\bf x}^c_1 \in \mathcal{C}^c_1$, ${\bf x}^c_1 \in \mathcal{C}^c_2$ (${\bf x}^c_1 \in \mathbb{R}^4$).}
\label{obj_fig}
\end{figure}

Inspired by the MSSR objective~\cite{cho-sdm04}, we first define the quality of the non-exhaustive, overlapping co-clusters by considering the sum of squared differences between each entry and each mean of the co-clusters the data point belongs to. This can be formally described as follows:
\begin{equation}
\label{neocc_entry}
\sum_{ \mathcal{C}_q \in g({\bf x}_j)} \sum_{ \mathcal{C}_p \in f({\bf x}_i)} \big( x_{ij} - \sum_{{\bf x}_t \in \mathcal{C}_q} \sum_{{\bf x}_s \in \mathcal{C}_p} \dfrac{x_{st}}{|\mathcal{C}_q||\mathcal{C}_p|} \big)^2
\end{equation}for $x_{ij}$ such that $f({\bf x}_i) \neq \emptyset$ and $g({\bf x}_j) \neq \emptyset$. Since each data point can belong to multiple row and column clusters, we need to consider all the combinations of these row and column clusters when we compute the mean of the co-clusters. 

Now, we represent the idea of (\ref{neocc_entry}) using vectors and matrices. Let $\mU = [u_{ij}]_{n \times k}$ denote the assignment matrix for row clustering, i.e., $u_{ij}=1$ if ${\bf x}_i$ belongs to cluster $j$; $u_{ij} =0$ otherwise. Similarly, let $\mV = [v_{ij}]_{m \times l}$ denote the assignment matrix for column clustering. Also, let $\hat{\mU} =$ $[\dfrac{{\bf u}_1}{\sqrt{n_1}}, \cdots, \dfrac{{\bf u}_k}{\sqrt{n_k}}]$ denote a normalized assignment matrix where ${\bf u}_c$ is the $c$-th column of $\mU$ and $n_c$ is the size of cluster $c$. Let ${\bf \hat{u}}_i$ denote the $i$-th column of $\hat{\mU}$. Similarly, we also define $\hat{\mV}$ and let ${\bf \hat{v}}_j$ denote the $j$-th column of $\hat{\mV}$. Let $\mathbbm{I}\{ exp \} = 1$ if $exp$ is true; 0 otherwise, and let $\mathbf{1}$ denote a vector having all the elements equal to one. Given a vector ${\bf y} \in \mathbb{R}^{m}$, let us define $D({\bf y})=[d_{ij}]_{m \times m}$ as the diagonal matrix with $d_{ii} = y_i$ ($i=1,\dots,m$). Then, our NEO-CC-M objective function is defined as follows:
\begin{equation}
\begin{array}{ll}
\underset{\mU, \mV} {\mbox{minimize}} & \sum\limits_{i=1}^k \sum\limits_{j=1}^l \| D({\bf u}_i) \mX D({\bf v}_j) - {\bf \hat{u}}_i {{\bf \hat{u}}_i}^T \mX {\bf \hat{v}}_j {{\bf \hat{v}}_j}^T \|_F^2 \\ 
\text{subject to} & \trace(\mU^T \mU) = (1+\alpha_r) n, \\ 
& \sum\nolimits_{i=1}^n \mathbbm{I}\{ (\mU\mathbf{1})_i = 0 \} \leq \beta_r n, \\ 
& \trace(\mV^T \mV) = (1+\alpha_c) m, \\
& \sum\nolimits_{i=1}^m \mathbbm{I}\{ (\mV\mathbf{1})_i = 0 \} \leq \beta_c m,
\end{array} 
\label{neocc_obj}
\end{equation}where $\alpha_r$ and $\beta_r$ are the parameters for row clustering, and $\alpha_c$ and $\beta_c$ are the parameters for column clustering. The parameters $\alpha_r$ and $\alpha_c$ control the amount of overlap among the clusters while $\beta_r$ and $\beta_c$ control the degree of non-exhaustiveness. We introduce these parameters motivated by the success of the one-way NEO-K-Means method~\cite{whang-sdm2015} where the similar parameters have been added to the traditional $k$-means objective.

The first two constraints in (\ref{neocc_obj}) are associated with the row clustering whereas the last two constraints are associated with the column clustering. The first constraint indicates that the total number of assignments in $\mU$ is equal to $(1+\alpha_r) n$. Thus, more than $n$ assignments are made in $\mU$ for $\alpha_r>0$, which implies that some data points belong to more than one cluster. The second constraint indicates the upper bound of the number of outliers, i.e., there can be at most $\beta_r n$ outliers. We can similarly interpret the last two constraints for the column clustering.

The NEO-CC-M objective seamlessly generalizes the NEO-K-Means~\cite{whang-sdm2015} and the MSSR objective~\cite{cho-sdm04}. If $\mV=\mI$, $\alpha_c=0$, $\beta_c=0$, then (\ref{neocc_obj}) is equivalent to the NEO-K-Means objective. If $\alpha_r=0$, $\alpha_c=0$, $\beta_r=0$, $\beta_c=0$, then (\ref{neocc_obj}) is equivalent to the MSSR objective.

Let us further explain the implication of the NEO-CC-M objective function. Suppose that we have a small data matrix $\mX \in \mathbb{R}^{4 \times 5}$ and the assignment matrices $\mU$ and $\mV$ as shown in Figure~\ref{obj_fig}(a). For an entry $x_{21}$, Figure~\ref{obj_fig}(b) shows the contribution of the entry $x_{21}$ to the NEO-CC-M objective in~(\ref{neocc_obj}) when ${\bf x}^r_2 \in \mathcal{C}^r_1$, ${\bf x}^r_2 \in \mathcal{C}^r_2$ (${\bf x}^r_2 \in \mathbb{R}^5$), ${\bf x}^c_1 \in \mathcal{C}^c_1$, ${\bf x}^c_1 \in \mathcal{C}^c_2$ (${\bf x}^c_1 \in \mathbb{R}^4$) where the superscripts indicate the row and column vectors and clusters. For the entry $x_{21}$, the NEO-CC-M objective takes account of four different means, each of which corresponds to a different combination of the row and column clusters.

\subsection{The NEO-CC-RCM Objective Function} 
Instead of considering the mean of each co-cluster, we can also consider the difference between each entry and the mean of the row clusters as well as the mean of the column clusters. Similar to how we encode the NEO-CC-M objective, we can formulate the element-wise difference as follows:

\begin{equation}
\label{neocc_entry2}
\sum_{ \mathcal{C}_q \in g({\bf x}_j)} \sum_{ \mathcal{C}_p \in f({\bf x}_i)} \big( x_{ij} - \sum_{{\bf x}_s \in \mathcal{C}_p} \dfrac{x_{sj}}{|\mathcal{C}_p|} - \sum_{{\bf x}_t \in \mathcal{C}_q} \dfrac{x_{it}}{|\mathcal{C}_q|} + \sum_{{\bf x}_t \in \mathcal{C}_q} \sum_{{\bf x}_s \in \mathcal{C}_p} \dfrac{x_{st}}{|\mathcal{C}_q||\mathcal{C}_p|} \big)^2
\end{equation}for $x_{ij}$ such that $f({\bf x}_i) \neq \emptyset$ and $g({\bf x}_j) \neq \emptyset$. Also, using matrices and vectors, we can represent (\ref{neocc_entry2}) as follows:
\begin{equation}
\begin{array}{ll}
\underset{\mU, \mV} {\mbox{minimize}} & \sum\limits_{i=1}^k \sum\limits_{j=1}^l \| \mH_{ij} \|_F^2, \text{where } \mH_{ij} \text{ is defined to be}\\
& D({\bf u}_i) \mX D({\bf v}_j) - {\bf \hat{u}}_i {{\bf \hat{u}}_i}^T \mX - \mX {\bf \hat{v}}_j {{\bf \hat{v}}_j}^T + {\bf \hat{u}}_i {{\bf \hat{u}}_i}^T \mX {\bf \hat{v}}_j {{\bf \hat{v}}_j}^T, \\
\text{subject to} & \trace(\mU^T \mU) = (1+\alpha_r) n, \\ 
& \sum\nolimits_{i=1}^n \mathbbm{I}\{ (\mU\mathbf{1})_i = 0 \} \leq \beta_r n, \\ 
& \trace(\mV^T \mV) = (1+\alpha_c) m, \\
& \sum\nolimits_{i=1}^m \mathbbm{I}\{ (\mV\mathbf{1})_i = 0 \} \leq \beta_c m.
\end{array} 
\label{neocc_obj2}
\end{equation}
If we let $\mH' =$ $\sum\limits_{i=1}^k \sum\limits_{j=1}^l \mH_{ij}$, we note that the row sum and the column sum of $\mH'$ is equal to zero (we add the co-cluster mean in~(\ref{neocc_entry2}) to retain this property). We name (\ref{neocc_obj2}) NEO-CC-RCM objective function. The idea of subtracting the row and the column means is similar to removing the user/item bias in recommender systems~\cite{koren}.

\subsection{Case Study}

We conduct a small case study to provide more insights about our NEO-CC objective functions. For simplicity, we focus on the NEO-CC-M objective function in this study, but our analysis can also be extended to the NEO-CC-RCM objective. 

Let us consider the following data matrix $\mX$:
\begin{equation}
\label{small_x}
\mX = \left[ \begin{array}{cccccc} 0.05 & 0.05 & 0.05 & 0 & 0 & 0 \\ 0.05 & 0.05 & 0.05 & 0 & 0 & 0 \\ 0.04 & 0.04 & 0.04 & 0 & 0.04 & 0.04 \\ 0.04 & 0.04 & 0 & 0.04 & 0.04 & 0.04 \\ 0 & 0 & 0 & 0.05 & 0.05 & 0.05 \\ 0 & 0 & 0 & 0.05 & 0.05 & 0.05 \\ 0 & 0 & 0.3 & 0 & 0 & 0 \end{array} \right]
\end{equation}
On $\mX$, let us consider four different co-clustering results and investigate how the NEO-CC-M objective function values are changed. Figure~\ref{obj_case} shows the NEO-CC-M objective values for the four different configurations. 

\begin{figure}[t]
\centering
\begin{minipage}[b]{1\linewidth}\centering
  \centering
  \begin{tabular}{cc}
    \subfloat[Objective value: 0.0720]{$\mU = \left[ \begin{array}{cc} 1 & 0 \\ 1 & 0 \\ 1 & 0 \\ 0 & 1 \\ 0 & 1 \\ 0 & 1 \\ 1 & 0 \end{array} \right] \mV = \left[ \begin{array}{cc} 1 & 0 \\ 1 & 0 \\ 1 & 0 \\ 0 & 1 \\ 0 & 1 \\ 0 & 1 \end{array} \right]$} 
    \subfloat[Objective value: 0.0677]{$\mU = \left[ \begin{array}{ccc} 1 & 0 & 0 \\ 1 & 0 & 0 \\ 0 & 1 & 0 \\ 0 & 1 & 0 \\ 0 & 0 & 1 \\ 0 & 0 & 1 \\ 1 & 0 & 0 \end{array} \right] \mV = \left[ \begin{array}{cc} 1 & 0 \\ 1 & 0 \\ 1 & 0 \\ 0 & 1 \\ 0 & 1 \\ 0 & 1 \end{array} \right]$} \\
    \subfloat[Objective value: 0.0137]{$\mU = \left[ \begin{array}{cc} 1 & 0 \\ 1 & 0 \\ 1 & 1 \\ 1 & 1 \\ 0 & 1 \\ 0 & 1 \\ 0 & 0 \end{array} \right] \mV = \left[ \begin{array}{cc} 1 & 0 \\ 1 & 0 \\ 1 & 0 \\ 0 & 1 \\ 0 & 1 \\ 0 & 1 \end{array} \right]$} 
    \subfloat[Objective value: 0.0102]{$\mU = \left[ \begin{array}{cc} 1 & 0 \\ 1 & 0 \\ 1 & 1 \\ 1 & 1 \\ 0 & 1 \\ 0 & 1 \\ 0 & 0 \end{array} \right] \mV = \left[ \begin{array}{cc} 1 & 0 \\ 1 & 0 \\ 0 & 0 \\ 0 & 1 \\ 0 & 1 \\ 0 & 1 \end{array} \right]$}
  \end{tabular}
\end{minipage}
\caption{NEO-CC-M objective values for four different co-clustering results}
\label{obj_case}
\end{figure}

In Figure~\ref{obj_case}(a)\&(b), we consider exhaustive, disjoint co-clusterings with different $k$. Notice that on these configurations, $nnz(\mU) + nnz(\mV)=13$ where $nnz(\mM)$ indicates the number of non-zero elements of matrix $\mM$. We observe that (b) yields a slightly small objective value than (a). 

Figure~\ref{obj_case}(c) implies that ${\bf x}^r_3$ and ${\bf x}^r_4$ belong to both of the row clusters, and ${\bf x}^r_7$ does not belong to any cluster. We see that (c) achieves a much smaller objective function value than (a) \& (b). Indeed, when we look at the data matrix $\mX$, we can see that (c) is a more desirable co-clustering result than (a) \& (b). 

Note that the NEO-CC-M objective value is proportional to the number of assignments in $\mU$ and $\mV$. In (c), we notice that $nnz(\mU) + nnz(\mV)=14$ ($[n\alpha_r]=1$, $[n\beta_r]=1$, $\alpha_c=0$, $\beta_c=0$). Thus, we also consider (d) where $nnz(\mU) + nnz(\mV)=13$. In this case, the column clustering becomes a disjoint and non-exhaustive clustering by setting $[n\alpha_r]=1$, $[n\beta_r]=1$, $[m\alpha_c]=-1$, $[m\beta_c]=1$. Notice that ${\bf x}^c_3$ is considered to be an outlier, so it does not belong to any cluster. We can see that (d) achieves the smallest objective value and also (d) might be the optimal co-clustering for $\mX$. 

On these examples, we observe that a smaller NEO-CC-M objective function value leads to a more desirable co-clustering result. This indicates that by optimizing our NEO-CC-M objective, we might be able to capture the natural co-clustering structures.

\setlength{\textfloatsep}{3pt}
\begin{algorithm}[htbp]
{\small
\caption{NEO-CC Algorithm}
\label{neocc_algo}
\begin{algorithmic}[1]
\renewcommand{\algorithmicrequire}{\textbf{Input:}}
\renewcommand{\algorithmicensure}{\textbf{Output:}}
\REQUIRE $\mX \in \mathbb{R}^{n \times m}$, $k$, $l$, $\alpha_r$, $\alpha_c$, $\beta_r$, $\beta_c$, $t_{max}$ 
\ENSURE Row clustering $\mU \in \{0,1\}^{n \times k} $, Column clustering $\mV \in \{0,1\}^{m \times l} $
\STATE Initialize $\mU$, $\mV$, and $t=0$.
\WHILE{not converged and $t<t_{max}$}
	\STATE /* Update Row Clustering */
	\FOR {\textbf{each} ${\bf x}_p \in \mathcal{X}^r$}
		\FOR {$q=1,\cdots,k$}
			\IF{NEO-CC-M objective is used}
			\STATE Compute $d^r_{pq}$ using (\ref{dpq}).
			\ELSE
			\STATE Compute $d^r_{pq}$ using (\ref{dpq2}).
			\ENDIF
		\ENDFOR
	\ENDFOR	
	\STATE Initialize $w=0$, $\mathcal{T}=\emptyset$, $\mathcal{S}=\emptyset$, and $\bar{\mathcal{C}}^r_i=\emptyset, \hat{\mathcal{C}}^r_i=\emptyset$ $\forall i$ $(i=1,\cdots,k)$.
	\WHILE{$w<(n+\alpha_r n)$}
		\IF{$w<(n-\beta_r n)$}
			\STATE Assign ${\bf x}^r_{i^*}$ to $\bar{\mathcal{C}}^r_{j^*}$ such that $(i^*,j^*)=$ \smash{$\underset{i,j}{\operatorname{argmin}}$} $d^r_{ij}$ where $\{(i,j)\} \notin \mathcal{T}, i \notin \mathcal{S}$. \label{argmin1}
			\STATE $\mathcal{S} = \mathcal{S} \cup \{i^*\}$.
		\ELSE
			\STATE Assign ${\bf x}^r_{i^*}$ to $\hat{\mathcal{C}}^r_{j^*}$ such that $(i^*,j^*)=$ \smash{$\underset{i,j}{\operatorname{argmin}}$} $d^r_{ij}$ where $\{(i,j)\} \notin \mathcal{T}$. \label{argmin2}
		\ENDIF
		\STATE $\mathcal{T} = \{(i^*,j^*)\} \cup \mathcal{T}$.
		\STATE $w=w+1$. 
	\ENDWHILE
	\STATE Update clusters $\mathcal{C}^r_i = \bar{\mathcal{C}}^r_{i} \cup \hat{\mathcal{C}}^r_{i}$ $\forall i$ $(i=1,\cdots,k)$. \label{update_row}
	\STATE /* Update Column Clustering */
	\FOR {\textbf{each} ${\bf x}_p \in \mathcal{X}^c$}
		\FOR {$q=1,\cdots,l$}
			\IF{NEO-CC-M objective is used}
			\STATE Compute $d^c_{pq}$ using (\ref{dpq3}).
			\ELSE
			\STATE Compute $d^c_{pq}$ using (\ref{dpq4}).
			\ENDIF
		\ENDFOR
	\ENDFOR	
	\STATE Initialize $w=0$, $\mathcal{T}=\emptyset$, $\mathcal{S}=\emptyset$, and $\bar{\mathcal{C}}^c_j=\emptyset, \hat{\mathcal{C}}^c_j=\emptyset$ $\forall j$ $(j=1,\cdots,l)$.
	\WHILE{$w<(m+\alpha_c m)$}
		\IF{$w<(m-\beta_c m)$}
			\STATE Assign ${\bf x}^c_{i^*}$ to $\bar{\mathcal{C}}^c_{j^*}$ such that $(i^*,j^*)=$ \smash{$\underset{i,j}{\operatorname{argmin}}$} $d^c_{ij}$ where $\{(i,j)\} \notin \mathcal{T}, i \notin \mathcal{S}$.
			\STATE $\mathcal{S} = \mathcal{S} \cup \{i^*\}$.
		\ELSE
			\STATE Assign ${\bf x}^c_{i^*}$ to $\hat{\mathcal{C}}^c_{j^*}$ such that $(i^*,j^*)=$ \smash{$\underset{i,j}{\operatorname{argmin}}$} $d^c_{ij}$ where $\{(i,j)\} \notin \mathcal{T}$.
		\ENDIF
		\STATE $\mathcal{T} = \{(i^*,j^*)\} \cup \mathcal{T}$.
		\STATE $w=w+1$. 
	\ENDWHILE
	\STATE Update clusters $\mathcal{C}^c_j = \bar{\mathcal{C}}^c_{j} \cup \hat{\mathcal{C}}^c_{j}$ $\forall j$ $(j=1,\cdots,l)$.
	\STATE $t=t+1$.
\ENDWHILE
\end{algorithmic}
}
\end{algorithm}

\section{The NEO-CC Algorithm}
\label{sec:algo}

To optimize our NEO-CC objective functions, we develop an efficient iterative algorithm which we call the NEO-CC algorithm described in Algorithm~\ref{neocc_algo}. 

The NEO-CC algorithm repeatedly updates $\mU$ (row clustering) and $\mV$ (column clustering) until the change in the objective becomes sufficiently small or the maximum number of iterations is reached. Within each iteration, $\mU$ and $\mV$ are alternatively updated. 

We can get reasonably good initializations of $\mU$ and $\mV$ by running the one-way iterative NEO-K-Means clustering algorithm on the data matrix and the transpose of the matrix, respectively. Similarly, we can also estimate the parameters $\alpha_r$, $\beta_r$, $\alpha_c$, and $\beta_c$ using the strategies suggested in~\cite{whang-sdm2015}. 

Algorithm~\ref{neocc_algo} consists of two main parts -- updating row clustering (lines 4--24) and updating column clustering (lines 26--46). Let us first describe how $\mU$ is updated. To update $\mU$, we need to compute distances between every data point in $\mathcal{X}^r$ and the clusters $\mathcal{C}^r_q$ for $q=1,\cdots,k$ (lines 4--12). Let $[d^r_{pq}]_{n \times k}$ denote these distances, and let $\mI_{p \cdot}$ denote the $p$-th row of the identity matrix of size $n$. If we use the NEO-CC-M objective, the distance between a data point ${\bf x}_p \in \mathcal{X}^r$ and a cluster $\mathcal{C}^r_q$ is computed by 
\begin{equation}
\label{dpq}
d^r_{pq} = \sum_{j=1}^{l} \norm{ (\mI_{p \cdot}) \mX D({\bf v}_j) - \dfrac{1}{\sqrt{\norm{{\bf u}_q}_1}}  {{\bf \hat{u}}_q}^T \mX {\bf \hat{v}}_j {{\bf \hat{v}}_j}^T }^2.
\end{equation} On the other hand, when we use the NEO-CC-RCM objective, the distance is computed by
\begin{equation}
\label{dpq2}
d^r_{pq} = \sum_{j=1}^{l} \norm{ (\mI_{p \cdot}) \mX \big( D({\bf v}_j) - {\bf \hat{v}}_j {{\bf \hat{v}}_j}^T \big) - \dfrac{1}{\sqrt{\norm{{\bf u}_q}_1}}  {{\bf \hat{u}}_q}^T \mX \big( D({\bf v}_j) - {\bf \hat{v}}_j {{\bf \hat{v}}_j}^T \big) }^2.
\end{equation} For each data point ${\bf x}_p \in \mathcal{X}^r$ ($p=1,\cdots,n$), we record its closest cluster and that distance. By sorting the data points in ascending order by the distance to its closest cluster, we assign the first $(1-\beta_r)n$ data points to their closest clusters to satisfy the second constraint in (\ref{neocc_obj}). Then, we make $\alpha_r n + \beta_r n$ assignments by taking $\alpha_r n + \beta_r n$ minimum distances among $[d^r_{pq}]_{n \times k}$. Note that, in total, we make $n + \alpha_r n$ assignments in $\mU$, which satisfies the first constraint in (\ref{neocc_obj}). In this way, the row clustering $\mU$ is updated. 

Similarly, we can update the column clustering $\mV$. Let $\mI_{\cdot p}$ denote the $p$-th column of the identity matrix of size $m$. The distance between a data point ${\bf x}_p \in \mathcal{X}^c$ and a column cluster $\mathcal{C}^c_q$ is computed by
\begin{equation}
\label{dpq3}
d^c_{pq} = \sum_{i=1}^{k} \norm{ D({\bf u}_i) \mX (\mI_{\cdot p}) - \dfrac{1}{\sqrt{\norm{{\bf v}_q}_1}} {\bf \hat{u}}_i {{\bf \hat{u}}_i}^T \mX {\bf \hat{v}}_q }^2
\end{equation}for the NEO-CC-M objective, and
\begin{equation}
\label{dpq4}
d^c_{pq} = \sum_{i=1}^{k} \norm{ \big( D({\bf u}_i) - {\bf \hat{u}}_i {{\bf \hat{u}}_i}^T \big) \mX (\mI_{\cdot p}) - \dfrac{1}{\sqrt{\norm{{\bf v}_q}_1}} \big( D({\bf u}_i) - {\bf \hat{u}}_i {{\bf \hat{u}}_i}^T \big) \mX {\bf \hat{v}}_q }^2 
\end{equation}for the NEO-CC-RCM objective. After computing the distances between every data point in $\mathcal{X}^c$ and the column clusters using (\ref{dpq3}) or (\ref{dpq4}), $\mV$ is updated in a similar way to how $\mU$ is updated.

We note that when we use the NEO-CC-M objective, if $\mV=\mI$, $\alpha_c=0$, $\beta_c=0$, then Algorithm~\ref{neocc_algo} is identical to the one-way NEO-K-Means clustering algorithm~\cite{whang-sdm2015}, and if $\alpha_r=0$, $\alpha_c=0$, $\beta_r=0$, $\beta_c=0$, then Algorithm~\ref{neocc_algo} is identical to the MSSR algorithm~\cite{cho-sdm04}. 

\subsection{Convergence Analysis}
\label{sec:conv}

We show that Algorithm~\ref{neocc_algo} monotonically decreases the NEO-CC-M and the NEO-CC-RCM objective functions, and thus converges in a finite number of iterations. We first state the following lemma which is need for the convergence proofs. 

\begin{mylem}
\label{lemma}
Let us consider the function $h({\bf z}) = \sum\limits_{i} \pi_i \norm{ {\bf a}_i - c{\bf z}\mM }^2_2  $ where ${\bf a}_i \in \mathbb{R}^{1 \times m}$, ${\bf z} \in \mathbb{R}^{1 \times m}$, $\pi_i > 0$ $\forall i$, $c= \dfrac{1}{\sqrt{\sum_i \pi_i}} $, and $\mM \in \mathbb{R}^{m \times m}$ such that $\mM\mM^T=\mM$. Let ${\bf z}^*$ denote the minimizer of $h({\bf z})$. Then, ${\bf z}^*$ is given by $ \Big( \sqrt{\sum_i \pi_i} \Big) \mM {{\bf z}^*}^T = \mM \Big( \sum\limits_i \pi_i {\bf a}_i^T \Big). $
\end{mylem}
\begin{proof}
We can express $h({\bf z})$ as follows:
$$h({\bf z}) = \sum\limits_{i} \pi_i \big( {\bf a}_i{\bf a}_i^T -2c{\bf z}\mM{\bf a}_i^T +c^2{\bf z}\mM\mM^T{\bf z}^T \big),$$
and the gradient is given by
$$ \dfrac{\partial h({\bf z})}{\partial {\bf z}} = \sum\limits_{i} \pi_i \big( -2c\mM{\bf a}_i^T + 2c^2\mM\mM^T{\bf z}^T \big).$$
By setting the gradient to zero, we get
\begin{align*}
\sum\limits_i \pi_i \mM {\bf a}_i^T &= c \Big( \sum\limits_i \pi_i \Big) \mM \mM^T {{\bf z}^*}^T \\
&= c \Big( \sum\limits_i \pi_i \Big) \mM {{\bf z}^*}^T \hspace{0.3cm} \text{since $\mM\mM^T=\mM$}
\end{align*}
By setting $c= \dfrac{1}{\sqrt{\sum_i \pi_i}}$, we get
$$ \Big( \sqrt{\sum_i \pi_i} \Big) \mM {{\bf z}^*}^T = \mM \Big( \sum\limits_i \pi_i {\bf a}_i^T \Big). $$ 
\end{proof}

Theorem~\ref{thm:conv} shows the convergence of Algorithm~\ref{neocc_algo} with the NEO-CC-M objective function.

\begin{mythm}
\label{thm:conv}
Algorithm~\ref{neocc_algo} monotonically decreases the NEO-CC-M objective function defined in~(\ref{neocc_obj}).
\end{mythm}

\begin{proof}
Let $J^{(t)}$ denote the NEO-CC-M objective (\ref{neocc_obj}) at $t$-th iteration. Let $\mU$ denote the assignment matrix of the current row clustering $\mathcal{C}$, and $\mU^*$ denote the assignment matrix of the updated row clustering $\mathcal{C}^*$ obtained by line \ref{update_row} in Algorithm~\ref{neocc_algo}.
\begin{align*}
J^{(t)} &= \sum\limits_{i=1}^k \sum\limits_{j=1}^l \norm{ D({\bf u}_i) \mX D({\bf v}_j) - {\bf \hat{u}}_i {{\bf \hat{u}}_i}^T \mX {\bf \hat{v}}_j {{\bf \hat{v}}_j}^T }_F^2 \\
&= \sum\limits_{i=1}^k \sum\limits_{j=1}^l \sum\limits_{ {\bf x}_p\in \mathcal{C}_i} \norm{ (\mI_{p \cdot}) \mX D({\bf v}_j) - \dfrac{1}{\sqrt{\norm{{\bf u}_i}_1}}  {{\bf \hat{u}}_i}^T \mX {\bf \hat{v}}_j {{\bf \hat{v}}_j}^T }_2^2 \\
&\geq \sum\limits_{i=1}^k \sum\limits_{j=1}^l \sum\limits_{ {\bf x}_p\in \mathcal{C}^*_i} \norm{ (\mI_{p \cdot}) \mX D({\bf v}_j) - \dfrac{1}{\sqrt{\norm{{\bf u}_i}_1}}  {{\bf \hat{u}}_i}^T \mX {\bf \hat{v}}_j {{\bf \hat{v}}_j}^T }_2^2 \hspace{0.3cm} \\ 
&\geq \sum\limits_{i=1}^k \sum\limits_{j=1}^l \sum\limits_{ {\bf x}_p\in \mathcal{C}^*_i} \norm{ (\mI_{p \cdot}) \mX D({\bf v}_j) - \dfrac{1}{\sqrt{\norm{{\bf u^*}_i}_1}}  {{\bf \hat{u^*}}_i}^T \mX {\bf \hat{v}}_j {{\bf \hat{v}}_j}^T }_2^2 \hspace{0.3cm} \\
&= \sum\limits_{i=1}^k \sum\limits_{j=1}^l \norm{ D({\bf u^*}_i) \mX D({\bf v}_j) - {\bf \hat{u^*}}_i {{\bf \hat{u^*}}_i}^T \mX {\bf \hat{v}}_j {{\bf \hat{v}}_j}^T }_F^2 \\
&\geq \sum\limits_{i=1}^k \sum\limits_{j=1}^l \norm{ D({\bf u^*}_i) \mX D({\bf v^*}_j) - {\bf \hat{u^*}}_i {{\bf \hat{u^*}}_i}^T \mX {\bf \hat{v^*}}_j {{\bf \hat{v^*}}_j}^T }_F^2 \hspace{0.3cm} \\
&= J^{(t+1)}
\end{align*}
The first inequality holds because we make assignments by line \ref{argmin1} \& line \ref{argmin2}, and the second inequality holds by Lemma~\ref{lemma} with ${\bf a}_i = (\mI_{p \cdot}) \mX D({\bf v}_j)$, $\mM = {\bf \hat{v}}_j {{\bf \hat{v}}_j}^T$, ${\bf z}^* = {{\bf \hat{u^*}}_i}^T \mX$, and $\sqrt{\sum_i \pi_i} = \sqrt{\norm{{\bf u^*}_i}_1}$. The last inequality indicates that we can similarly show the decrease from $\mV$ to $\mV^*$.
\end{proof}

Similarly, we can also show the convergence of Algorithm~\ref{neocc_algo} with the NEO-CC-RCM objective function.

\begin{mythm}
\label{thm:conv2}
The NEO-CC Algorithm monotonically decreases the NEO-CC-RCM objective function defined in (\ref{neocc_obj2}).
\end{mythm}

\begin{proof}
Let $J^{(t)}$ denote the NEO-CC-RCM objective at $t$-th iteration.
\begin{align*}
J^{(t)} &= \sum\limits_{i=1}^k \sum\limits_{j=1}^l \| \mH_{ij} \|_F^2, \text{ where } \\
& \mH_{ij} = D({\bf u}_i) \mX D({\bf v}_j) - {\bf \hat{u}}_i {{\bf \hat{u}}_i}^T \mX - \mX {\bf \hat{v}}_j {{\bf \hat{v}}_j}^T + {\bf \hat{u}}_i {{\bf \hat{u}}_i}^T \mX {\bf \hat{v}}_j {{\bf \hat{v}}_j}^T\\
&= \sum\limits_{i=1}^k \sum\limits_{j=1}^l \sum\limits_{ {\bf x}_p\in \mathcal{C}_i} \norm{ \mH^p_{ij} }_2^2, \text{ where } \\
& \mH^p_{ij} = (\mI_{p \cdot}) \mX \big( D({\bf v}_j) - {\bf \hat{v}}_j {{\bf \hat{v}}_j}^T \big) - \dfrac{1}{\sqrt{\norm{{\bf u}_i}_1}}  {{\bf \hat{u}}_i}^T \mX \big( D({\bf v}_j) - {\bf \hat{v}}_j {{\bf \hat{v}}_j}^T \big)\\
&\geq \sum\limits_{i=1}^k \sum\limits_{j=1}^l \sum\limits_{ {\bf x}_p\in \mathcal{C}^*_i} \norm{ \mH^p_{ij} }_2^2 
\geq \sum\limits_{i=1}^k \sum\limits_{j=1}^l \sum\limits_{ {\bf x}_p\in \mathcal{C}^*_i} \norm{ \mH^p_{i^*j} }_2^2 
= \sum\limits_{i=1}^k \sum\limits_{j=1}^l \norm{ \mH_{i^*j} }_F^2 \\
&\geq \sum\limits_{i=1}^k \sum\limits_{j=1}^l \norm{ \mH_{i^*j^*} }_F^2 
= J^{(t+1)}
\end{align*}
The first inequality holds because we make assignments by line 16 \& line 19, and the second inequality holds by Lemma 1 with $\mM = D({\bf v}_j) - {\bf \hat{v}}_j {{\bf \hat{v}}_j}^T$, $\sqrt{\sum_i \pi_i} = \sqrt{\norm{{\bf u^*}_i}_1}$, ${\bf a}_i = (\mI_{p \cdot}) \mX \mM$, and ${\bf z}^* = {{\bf \hat{u^*}}_i}^T \mX$. The last inequality indicates the decrease of the objective by updating column clustering from $\mV$ to $\mV^*$.
\end{proof}

\begin{figure}[t]
\centering
\begin{minipage}[b]{0.8\linewidth}\centering
  \centering
  \begin{tabular}{cc}
    \subfloat[The NEO-CC-M objective vs. Iterations]{\includegraphics[width=0.5\textwidth]{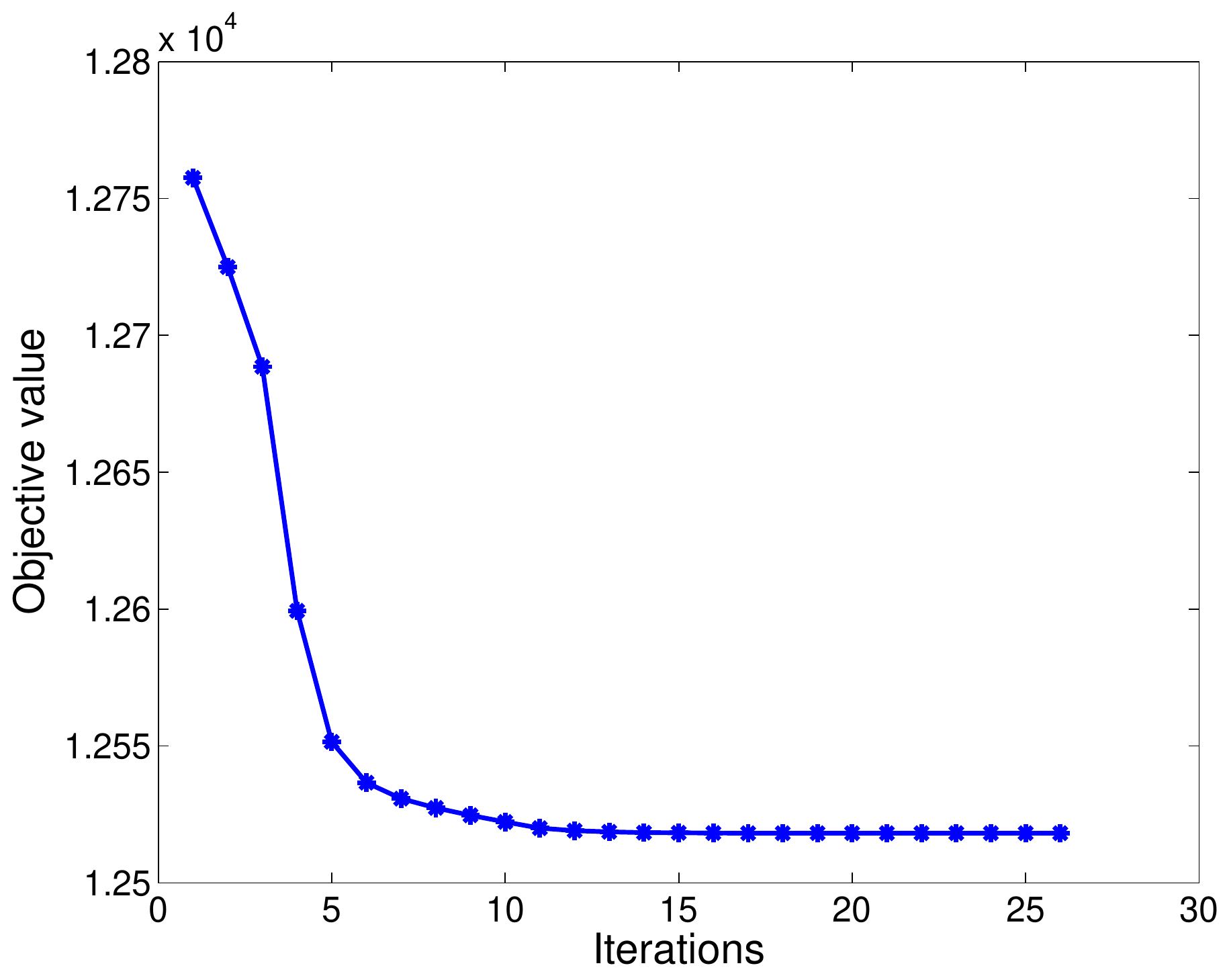}}&
    \subfloat[The NEO-CC-RCM objective vs. Iterations]{\includegraphics[width=0.5\textwidth]{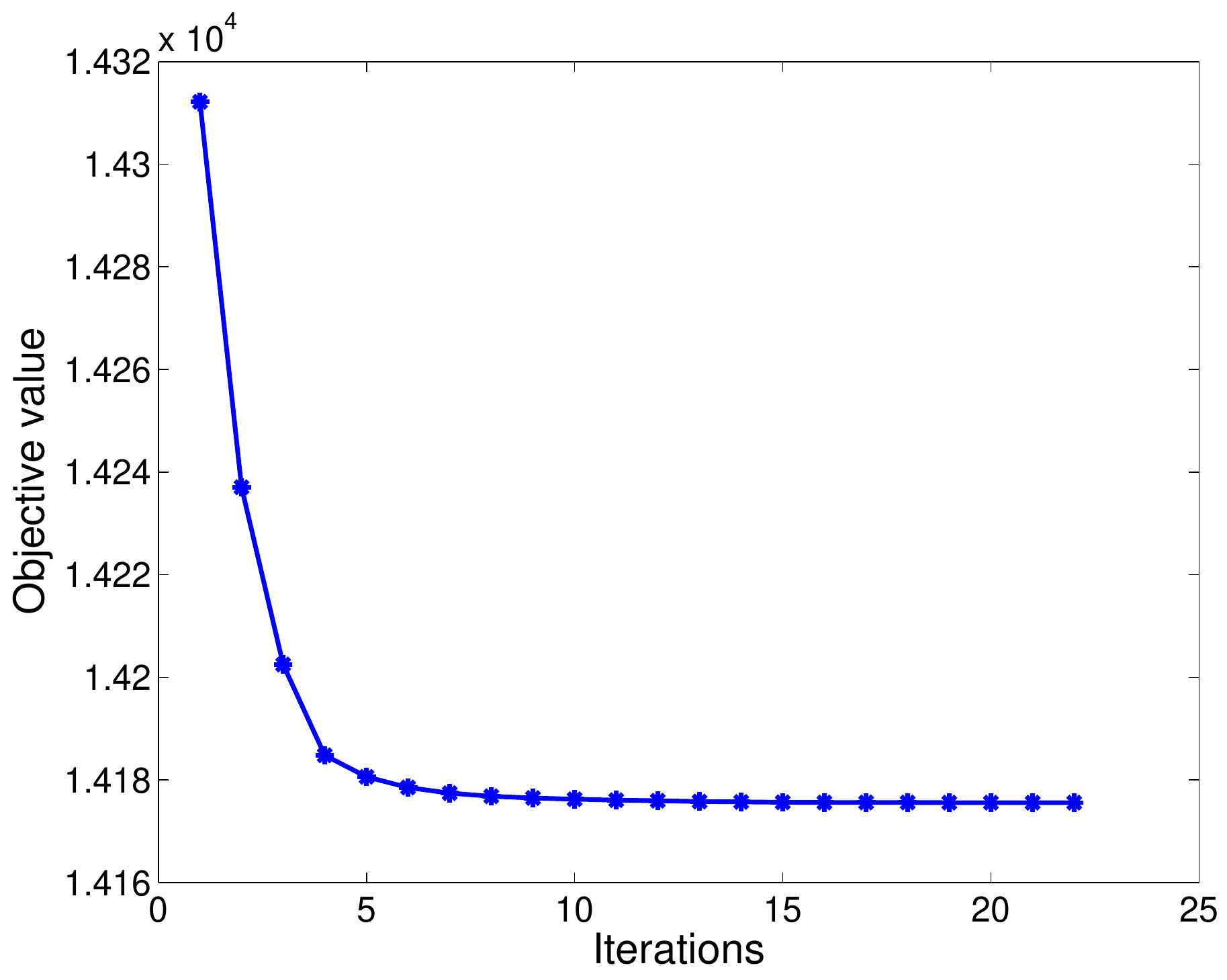}}
  \end{tabular}
\end{minipage}
\caption{The NEO-CC algorithm monotonically decreases the NEO-CC objective values.}
\label{fig:obj}
\end{figure}

We now experimentally show the convergence of the NEO-CC algorithm. In Figure~\ref{fig:obj}(a)\&(b), we present the change of the NEO-CC objective function values on the yeast gene expression dataset (Details about this dataset are described in Section~\ref{sec:exp}). The $x$-axis represents the number of iterations, and the $y$-axis represents the objective function value. We see that the the NEO-CC objective function values are monotonically decreased. 

\section{Experimental Results}
\label{sec:exp}
We compare the performance of our NEO-CC algorithm with other state-of-the-art co-clustering and one-way clustering algorithms (which are listed below) on real-world datasets. 

\begin{itemize}
\item The ROCC algorithm~\cite{deodhar-icml09} (ROCC): ROCC can detect non-exhaustive, overlapping co-clusters in noisy datasets. We used the software provided by the authors of~\cite{deodhar-icml09}.  
\item The infinite plaid model~\cite{infp} (IPM): IPM is also able to discover non-exhaustive, overlapping co-clusters. We used the Matlab implementation of the method\footnote{We downloaded the code from \url{https://github.com/k-ishiguro/InfinitePlaidModels}}.
\item The minimum sum-squared residue co-clustering~\cite{cho-sdm04} (MSSR): MSSR is designed to find exhaustive and disjoint co-clusters. In~\cite{cho-sdm04}, two objective functions are proposed, each of which is related to our NEO-CC-M and NEO-CC-RCM objective. In our experiments, the two MSSR objective functions are denoted by MSSR1 and MSSR2.
\item The iterative one-way NEO-K-Means algorithm~\cite{whang-sdm2015} (NEO-iter): NEO-iter is an iterative one-way clustering algorithm that produces non-exhaustive, overlapping clusters. 
\item The LRSDP-based one-way NEO-K-Means algorithm~\cite{hou-whang-kdd2015} (NEO-lrsdp): NEO-lrsdp is a low-rank SDP(semidefinite programming)-based one-way clustering algorithm which is an improved version of the iterative NEO-K-Means.
\end{itemize}

We run the NEO-CC algorithm with our two different objectives, NEO-CC-M and NEO-CC-RCM. 

\subsection{User-Movie Ratings Datasets}

\begin{table*}[htbp]
\caption{$F_1$ scores (\%) on three MovieLens datasets. NEO-CC-M and NEO-CC-RCM methods achieve higher $F_1$ scores than other methods.} 
\label{tb:f1}
\centering
{\small
\begin{tabularx}{\linewidth}{llXXXXXXXX}
\toprule
	   & & IPM & ROCC & MSSR1 & MSSR2 & NEO-iter & NEO-lrsdp & NEO-CC-M & NEO-CC-RCM \\
\midrule
\multirow{4}{*}{ML1} & Trial 1-5& Table~\ref{tb:ipm} & 55.7 & 43.8 & 44.2 & 56.3 & 56.4 & 58.1 & {\bf 58.4} \\ \cline{2-10}
& best& 37.0 & 55.7 & 43.8 & 44.2 & 56.3 & 56.4 & 58.1 & {\bf 58.4} \\
& worst& 11.8 & 55.7 & 43.8 & 44.2 & 56.3 & 56.4 & 58.1 & {\bf 58.4} \\
& avg.$\pm$std & 22.4$\pm$10.8 & 55.7$\pm$0 & 43.8$\pm$0 & 44.2$\pm$0 & 56.3$\pm$0 & 56.4$\pm$0 & 58.1$\pm$0 & 58.4$\pm$0\\ \midrule
\multirow{8}{*}{ML2} & Trial 1& 31.2 & 53.3 & 50.2 & 49.8 & 56.8 & 56.8 & {\bf 58.1} & 57.7 \\
& Trial 2& 36.2 & 53.3 & 50.5 & 48.2 & 56.8 & 56.8 & {\bf 58.8} & 57.2 \\
& Trial 3& 18.6 & 53.3 & 50.6 & 50.2 & 56.8 & 56.8 & {\bf 58.1} & 58.0 \\
& Trial 4& 25.4 & 53.3 & 50.5 & 50.5 & 56.8 & 56.8 & {\bf 58.6} & 56.6 \\
& Trial 5& 22.0 & 53.3 & 50.5 & 48.3 & 56.8 & 56.8 & {\bf 58.7} & 57.8 \\ \cline{2-10}
& best & 36.2 & 53.3 & 50.6 & 50.5 & 56.8 & 56.8 & {\bf 58.8} & 58.0 \\
& worst & 18.6 & 53.3 & 50.2 & 48.2 & 56.8 & 56.8 & {\bf 58.1} & 56.6 \\
& avg.$\pm$std & 26.7$\pm$7.1 & 53.3$\pm$0 & 50.5$\pm$0.1 & 49.4$\pm$1.1 & 56.8$\pm$0 & 56.8$\pm$0 & 58.4$\pm$0.3 & 57.5$\pm$0.5 \\
\midrule
\multirow{8}{*}{ML3} & Trial 1& 37.4 & 55.3 & 47.2 & 46.3 & 57.0 & 56.8 & 58.5 & {\bf 59.7} \\
& Trial 2& 40.5 & 55.3 & 47.4 & 46.4 & 57.0 & 56.8 & 58.4 & {\bf 61.5} \\
& Trial 3& 25.9 & 55.3 & 47.3 & 46.4 & 57.0 & 56.8 & 58.6 & {\bf 61.5} \\
& Trial 4& 38.4 & 55.3 & 47.4 & 46.2 & 57.0 & 56.8 & 58.6 & {\bf 59.4} \\
& Trial 5& 31.3 & 55.5 & 47.1 & 46.4 & 59.1 & 59.3 & {\bf 59.8} & 57.2 \\ \cline{2-10}
& best & 40.5 & 55.5 & 47.4 & 46.4 & 59.1 & 59.3 & 59.8 & {\bf 61.5} \\
& worst & 25.9 & 55.3 & 47.1 & 46.2 & 57.0 & 56.8 & {\bf 58.4} & 57.2 \\
& avg.$\pm$std & 34.7$\pm$6.0 & 55.4$\pm$0.1 & 47.3$\pm$0.2 & 46.3$\pm$0.1 & 57.4$\pm$0.9 & 57.3$\pm$1.1 & 58.8$\pm$0.6 & 59.9$\pm$1.8 \\
\bottomrule
\end{tabularx}}
\end{table*}

\begin{table}
\caption{$F_1$ scores (\%) of IPM on ML1.}
\label{tb:ipm}
\begin{tabularx}{1.04\linewidth}{XXXXXX}
\toprule
Method & Trial 1 & Trial 2 & Trial 3 & Trial 4 & Trial 5  \\ \midrule
IPM & 37.0 & 29.5 & 12.7 & 21.1 & 11.8 \\
\bottomrule
\end{tabularx}
\end{table}

We test the clustering performance on user-movie ratings datasets from MovieLens~\cite{movie}. (MovieLens is a web site where users rate some movies they watched and the system recommends movies for the users based on collaborative filtering.) In a user-movie ratings data matrix, the element of the $i$-th row and the $j$-th column indicates the rating of the $j$-th movie from the $i$-th user. A user can rate each movie by assigning an integer value from one to five. In the dataset, we have information about the genres of the movies (e.g., action, romance, comedy, mystery, etc.). By treating the genres as clusters~\cite{mocc}, we can create the ground-truth clusters for the movies. Since a movie usually belongs to multiple genres, there exists overlaps among the clusters. We create three different datasets (denoted by ML\#) by selecting different kinds of genres. ML1 contains 44,269 ratings, ML2 contains 28,335 ratings, and ML3 contains 30,769 ratings.

We compare the clustering performance by comparing the algorithmic clusters and the ground-truth clusters. Since we have the ground-truth clusters only for the movies, we focus on matching these ground-truth clusters with the algorithmic column clusters of the data matrix. We repeat the experiments for 5 times on each dataset, and compute $F_1$ scores to measure the similarity between the algorithmic clusters and the ground-truth clusters (see~\cite{whang-sdm2015} to check how to compute the $F_1$ scores). Higher $F_1$ score indicates better clusters. In Table~\ref{tb:f1}, we show the $F_1$ scores (\%) of each method for each trial and also present the best, the worst, the average and the standard deviation of the scores. On ML1 dataset, all the methods except IPM produce identical results for the five trials. For the IPM results on this dataset, see Table~\ref{tb:ipm}. Figure~\ref{fig:ml123} shows the $F_1$ scores of the best baseline method (NEO-lrsdp) and the NEO-CC algorithm on the three ML datasets.

\begin{figure}[t]
\centering
\begin{minipage}[b]{1\linewidth}\centering
  \centering
  \begin{tabular}{cc}
  \subfloat[The ML1 and ML3 datasets]{\includegraphics[width=0.48\textwidth]{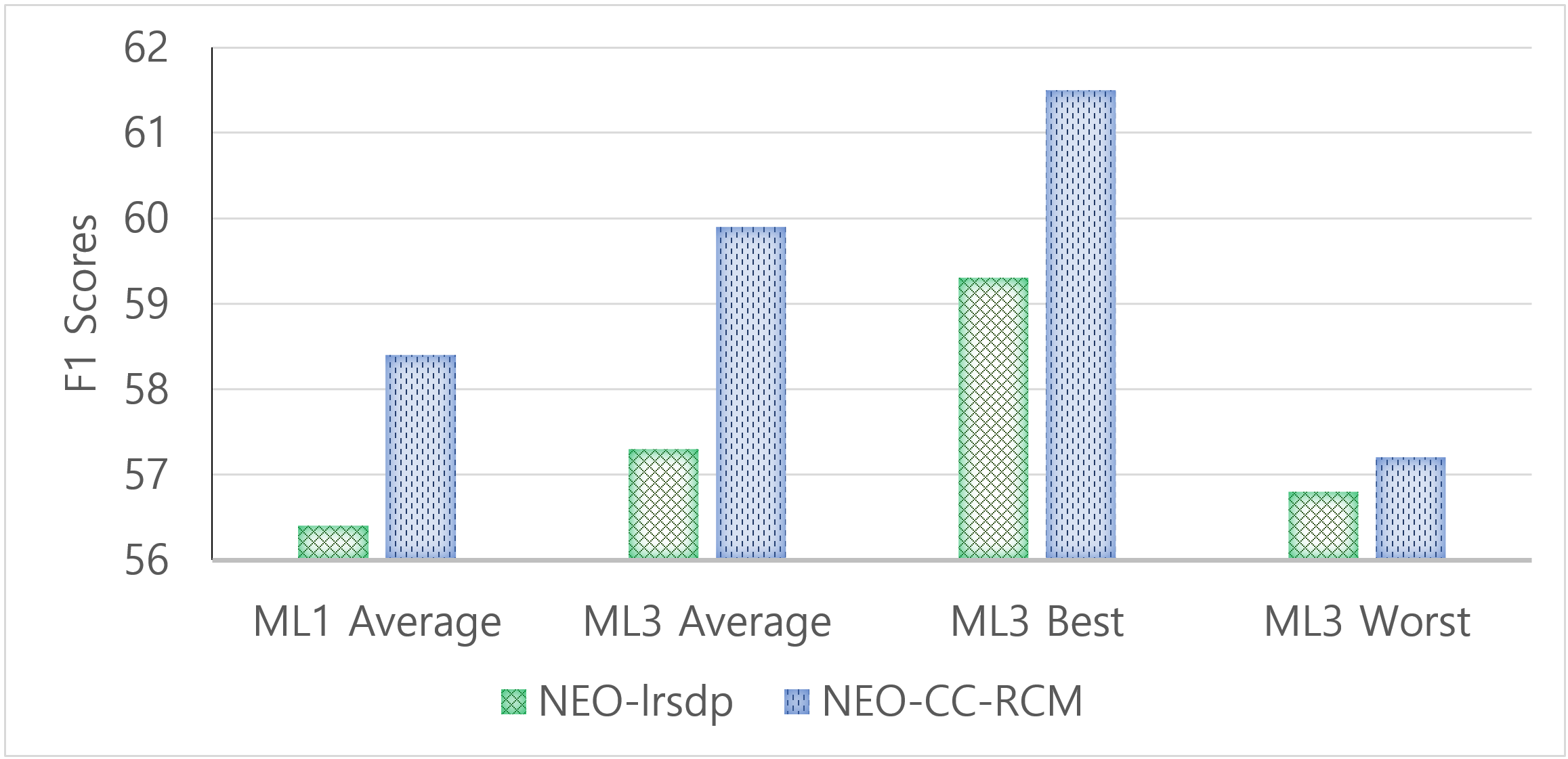}} &
  \subfloat[The ML2 dataset]{\includegraphics[width=0.48\textwidth]{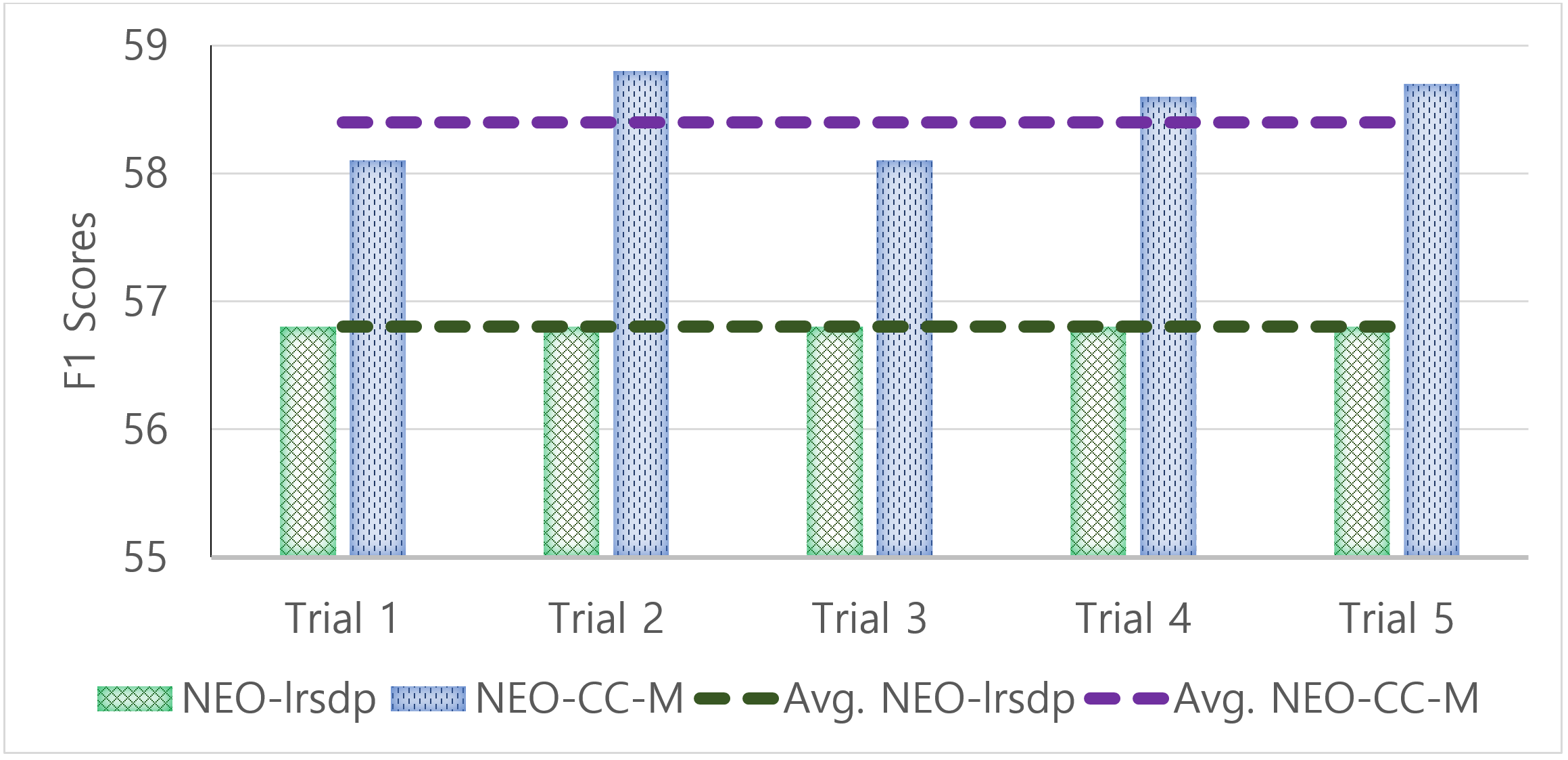}} 
  \end{tabular}
\end{minipage}
\caption{$F_1$ scores of the best baseline method (NEO-lrsdp) and the NEO-CC algorithm on the three ML datasets.}
\label{fig:ml123}
\end{figure}

\begin{table*}[t]
\caption{$F_1$ scores (\%) on the yeast gene expression dataset.}
\label{tb:yeast}
\begin{tabularx}{1.04\linewidth}{lXXXXXXXXX}
\toprule
Method & Trial 1 & Trial 2 & Trial 3 & Trial 4 & Trial 5 & best & worst & avg. & std. \\ \midrule
IPM & N/A & N/A & N/A & N/A & N/A & N/A & N/A & N/A & N/A\\
ROCC & 14.9 & 12.8 & 15.0 & 14.2 & 14.9 & 15.0 & 12.8 & 14.3 & 0.9 \\
MSSR1 & 17.1 & 17.0 & 16.5 & 16.4 & 17.4 & 17.4 & 16.4 & 16.9 & 0.4\\
MSSR2 & 19.3 & 18.2 & 18.5 & 18.5 & 18.0 & 19.3 & 18.0 & 18.5 & 0.5\\
NEO-iter & 35.9 & 36.6 & 36.0 & 36.0 & 35.6 & 36.6 & 35.6 & 36.0 & 0.3\\
NEO-lrsdp & 39.0 & 39.1 & 39.0 & {\bf 39.1} & 39.1 & 39.1 & {\bf 39.0} & 39.1 & 0\\
NEO-CC-M & {\bf 40.6} & {\bf 40.6} & {\bf 40.6} & 36.2 & {\bf 40.7} & {\bf 40.7} & 36.2 & {\bf 40.0} & 2\\
NEO-CC-RCM & 36.0 & 37.9 & 35.9 & 35.6 & 38.1 & 38.1 & 35.6 & 36.7 & 1.2\\
\bottomrule
\end{tabularx}
\end{table*}

We see that the NEO-CC algorithm (both NEO-CC-M and NEO-CC-RCM) outperforms all the baseline methods. Among the baseline methods, NEO-lrsdp achieves the highest $F_1$ score. Since the underlying ground-truth clusters contain overlaps between the clusters, the accuracies of the MSSR methods are relatively low (MSSR generates pairwise disjoint co-clusters). We note that the IPM method underperforms even though the method is able to produce non-exhaustive, overlapping co-clusters. It is interesting to see that the performance of NEO-CC is even better than NEO-lrsdp because we know that the NEO-lrsdp method involves much more expensive operations than the NEO-CC method which is a simple iterative co-clustering algorithm. 

\subsection{Yeast Gene Expression Data}
It has been known that co-clustering is a useful tool for analyzing gene expression data. We get an yeast gene expression dataset from~\cite{yeast}. In this data matrix, each row represents a gene, and each column represents an expression level under a particular biological condition. By clustering or co-clustering this gene expression data, we can group genes with similar functions~\cite{eisen},~\cite{cheng-ismb00}. The yeast dataset has been studied in the context of multi-labelled classification, and is known to be a difficult dataset~\cite{yeast}. In this dataset, each gene can belong to multiple functional classes. For example, a gene can be associated with cell growth, cell division, and cell organization. We treat each functional class as a ground-truth cluster. There are 2,417 genes, 103 features (expressions), and 14 functional classes.

Table~\ref{tb:yeast} shows the $F_1$ scores of the methods on the yeast dataset. The performance of MSSR is not good because the dataset contains a large overlap among the clusters. The IPM method failed to process the yeast dataset. We see that the NEO-* methods significantly outperform the other methods (IPM, ROCC, MSSR1, and MSSR2).  Overall, the NEO-CC-M method shows the best performance. The second best method is NEO-lrsdp. Among five trials, NEO-CC-M outperforms NEO-lrsdp four times. But, we note that NEO-lrsdp is a highly-tuned one-way clustering method whereas the NEO-CC-M is a simple iterative algorithm. We expect that we can further improve the performance of the NEO-CC algorithm by adapting an SDP-based approach, which is one of our future works.

\section{Conclusions \& Future Work}
We formulate the non-exhaustive, overlapping co-clustering problem and propose intuitive objective functions to solve the problem. To optimize the objectives, we develop a simple iterative algorithm which we call the NEO-CC algorithm. We prove that the NEO-CC algorithm monotonically decreases our new co-clustering objective functions. Experimental results show that our NEO-CC algorithm is able to identify qualitatively better clusters than other state-of-the-art co-clustering and one-sided clustering methods. The NEO-CC method provides a principled way to effectively capture the underlying co-clustering structure of real-world data. We plan to extend our NEO-CC method to solve various types of clustering tasks~\cite{whang-pvldb2020} and develop more scalable solution procedures to allow the NEO-CC idea to be also utilized in large-scale graph analysis tasks~\cite{whang-icdm2012},~\cite{whang-tkde2016}.

\section*{Acknowledgments}
This research was supported by Basic Science Research Program through the NRF of Korea funded by MOE(2016R1D1A1B03934766 and NRF-2010-0020210) and by the National Program for Excellence in SW supervised by the IITP(2015-0-00914), Korea to JW, and by NSF grants CCF-1320746 and IIS-1546452 to ID. J. Whang is the corresponding author.

\bibliographystyle{abbrv}
\bibliography{ref_neocc}

\end{document}